\documentclass[nohyperref]{article}

\usepackage{microtype}
\usepackage{graphicx}
\usepackage{subfigure}
\usepackage{booktabs} 

\usepackage{hyperref}

\usepackage[accepted]{icml2022}

\usepackage{amsmath}
\usepackage{amssymb}
\usepackage{mathtools}
\usepackage{amsthm}

\usepackage[capitalize,noabbrev]{cleveref}

\usepackage{nicefrac}
\floatname{algorithm}{Algorithm}
\renewcommand{\algorithmicrequire}{\textbf{Input:}}
\renewcommand{\algorithmicensure}{\textbf{return}}
\renewcommand{\algorithmicfunction}{\textbf{def}}
\usepackage{eqparbox}
\renewcommand{\algorithmiccomment}[1]{\hfill\eqparbox{COMMENT}{\# #1}}
\usepackage[normalem]{ulem}

\theoremstyle{plain}
\newtheorem{theorem}{Theorem}[section]

\theoremstyle{definition}

\theoremstyle{remark}

\usepackage[textsize=tiny]{todonotes}

\icmltitlerunning{Path-Gradient Estimators for Continuous Normalizing Flows}

\begin{document}

\twocolumn[
\icmltitle{Path-Gradient Estimators for Continuous Normalizing Flows}




\begin{icmlauthorlist}
\icmlauthor{Lorenz Vaitl}{tu}
\icmlauthor{Kim A. Nicoli}{tu,bif}
\icmlauthor{Shinichi Nakajima}{tu,bif,riken}
\icmlauthor{Pan Kessel}{tu,bif}
\end{icmlauthorlist}

\icmlaffiliation{tu}{Machine Learning Group, Department of Electrical Engineering \& Computer Science, Technische Universit\"at Berlin, Germany}
\icmlaffiliation{bif}{BIFOLD - Berlin Institute for the Foundations of Learning and Data, Technische Universit\"at Berlin, Berlin, Germany}
\icmlaffiliation{riken}{RIKEN Center for AIP, 103-0027 Tokyo, Chuo City, Japan}

\icmlcorrespondingauthor{Pan Kessel}{pan.kessel@tu-berlin.de}

\icmlkeywords{Normalizing Flows, Continuous Normalizing Flows, Path Gradient Estimator, Sticking the landing, Variational Autoencoder, Variational Inference, Kullback-Leibler Divergence}

\vskip 0.3in
]



\printAffiliationsAndNotice{}  

\begin{abstract}
Recent work has established a path-gradient estimator for simple variational Gaussian distributions and has argued that the path-gradient is particularly beneficial in the regime in which the variational distribution approaches the exact target distribution. In many applications, this regime can however not be reached by a simple Gaussian variational distribution. In this work, we overcome this crucial limitation by proposing a path-gradient estimator for the considerably more expressive variational family of continuous normalizing flows. We outline an efficient algorithm to calculate this estimator and establish its superior performance empirically. 
\end{abstract}

\section{Introduction}
Variational Inference is increasingly applied to very high-dimensional systems in various application domains. For example, recent developments in Variational Autoencoders (VAEs) \citep{kingma2013auto, rezende2014stochastic, kingma2016improved, vdberg2018sylvester, chen2018neural} have led to models that can generate high-resolution natural images. Similarly promising examples are applications in computer graphics \citep{muller2019neural}, theoretical physics \citep{wu2019solving, nicoli2020asymptotically, albergo2021introduction, nicoli2021estimation} and computational chemistry \citep{noe2019boltzmann}.

 A recurrent theme of most of these examples is that a highly expressive normalizing flow model is used as the variational density which can be trained by maximizing the evidence lower bound (ELBO) with the reparameterization trick \citep{kingma2013auto, rezende2014stochastic}. These models are considerably more expressive than traditional models, such as (mixtures of) Gaussians, and thereby enable the application of Variational Inference to these highly complex problems. 
 
 A promising subclass of these models are continuous normalizing flows \citep{chen2018neural} because they allow for a free-form Jacobian \citep{grathwohl2018scalable} and are particularly suited for incorporating known symmetries of the target density \citep{kohler2020equivariant} as inductive bias in their architecture. The latter point is of crucial importance in many applications in the natural sciences, (see e.g. \citet{kohler2020equivariant, kanwar2020equivariant, haan2021scaling}).

In a seminal paper, \citet{roeder2017sticking} proposed a promising alternative training method for Variational Inference. Specifically, they derived a new unbiased estimator for the gradient of the ELBO with respect to the variational parameters. Unlike the standard estimator, their proposed estimator is not based on the total derivative but rather on the path gradient, i.e. it only takes into account the implicit dependency on the variational parameters through reparameterized sampling but is insensitive to any explicit parameter dependency. From theoretical arguments, which we will discuss in Section~\ref{sec:vi}, it may be expected that this estimator has lower variance than the standard estimator. In their experiments, \citet{roeder2017sticking} demonstrated that, for comparatively simple variational models, this estimator indeed leads to faster convergences and better approximation results.

Unfortunately, their method cannot be efficiently implemented for normalizing flows for reasons that we will describe in detail in Section~\ref{sec:pathgradcnf}. As a result, the path-gradient estimator cannot efficiently be applied to the modern normalizing-flow-based applications of Variational Inference discussed above.

The main contribution of this work is to propose a method to estimate the path-gradient for continuous normalizing flow models. Our method has only slightly larger runtime and the same memory costs per iteration as the standard total derivative training. This mild increase in computational cost per iteration is however more than compensated by faster convergence often leading to a substantial speed-up in training time and better overall training results.

We demonstrate for two application domains, i.e. VAEs and lattice field theories in theoretical physics, that a simple replacement of the standard total gradient by the path-gradient estimator improves the performance across different architectures of continuous normalizing flows, datasets as well as for fixed and adaptive step-size ODE solvers.

\section{Related Work}
Path gradients were introduced by \citet{roeder2017sticking}. In this reference, the authors propose an algorithm to calculate path-gradients for certain simple variational families and demonstrate its superior performance on various VAE tasks. 
In \citet{tucker2018doubly}, the authors extended the results of \citet{roeder2017sticking} to other VAE losses, such as Importance Weighted Autoencoder \citep{burda2015importance}, Reweighted Wake Sleep \citep{bornschein2014reweighted}, and Jackknife Variational Autoencoder \citep{nowozin2018debiasing}. Later work \citep{finke2019importanceweighted, geffner2020difficulty, geffner2021empirical, bauer2021generalized} extended these results to further VAE loss functions, for example based on $\alpha$-divergences, and clarified theoretical aspects of the original references.

We extend on these reference by proposing a method to estimate the path-gradient which can efficiently be implemented for Normalizing Flows. To the best of our knowledge, \citet{agrawal2020advances} is the only reference studying path-wise gradient estimators of normalizing flows as part of a broader ablation study for comparatively simple models from the STAN library. In comparison to our method, their proposal has twice the memory and substantially larger runtime costs. This severely restricts their applicability in more complex problem settings. We refer to Section~\ref{sec:pathgradcnf} for a more detailed discussion.

\section{Variational Inference with Path-Gradients}\label{sec:vi}
In Variational Inference, we aim to approximate a target density $p$ by a variational density $q_\theta$ with parameters $\theta$. We denote the sampling space of these densities by $\mathcal{X}$. Typically, the target density $p$ is only known in its unnormalized form $\hat{p}$ with 
\begin{align*}
    p(x) = \frac{1}{Z} \hat{p}(x) \,,
\end{align*}
while the partition function $Z$ is computationally intractable. 

This specific type of Variational Inference arises in many application domains such as Variational Autoencoders (VAEs) \citep{kingma2013auto, rezende2014stochastic}, Boltzmann generators in Quantum Chemistry \citep{noe2019boltzmann, kohler2020equivariant}, as well as Generalized Neural Samplers in Statistical Physics \citep{wu2019solving, nicoli2020asymptotically} and Lattice Field Theories \citep{kanwar2020equivariant, albergo2021introduction, nicoli2021estimation, haan2021scaling}.

In the following, we will assume that the variational density $q_\theta(x)$ is obtained by reparameterization $g_\theta: \mathcal{Z} \to \mathcal{X}$ of a base density $q_Z$ with sampling space $\mathcal{Z}$, i.e. any sample $x \in \mathcal{X}$ can be written as
\begin{align*}
    x = g_\theta(z)
\end{align*}
for some sample $z \in \mathcal{Z}$ from the base density $q_Z$. An example is the variational family of Gaussians $\mathcal{N}(\mu, \sigma^2)$ parameterized by the mean $\mu$ and the variance $\sigma^2$ with base density $q_Z = \mathcal{N}(0,1)$ and 
\begin{align*}
    g_{\theta} (z) = \sigma z + \mu \,, && \textrm{with} && \theta \in \{\mu, \sigma^2\} \,.
\end{align*}
Another example are normalizing flows for which $g_\theta$ is given by an invertible neural network. In particular, reparameterized densities allow for the reparameterization trick
\begin{align}
    \mathbb{E}_{x \sim q_\theta} \left[ f(x) \right] &= \mathbb{E}_{z \sim q_z} \left[ f(g_\theta (z) ) \right] \, \label{eq:reparamTrick}
\end{align}
where $f:\mathcal{X} \to \mathbb{R}$ is an arbitrary function. 

The most widely used optimization criterion in Variational Inference is obtained by minimization of the reverse Kullback--Leibler divergence
$\textrm{KL}(q_\theta, p) = \mathbb{E}_{x \sim q_\theta} \left[ \ln \frac{q_\theta(x)}{p(x)} \right]$. Crucially, its derivative  with respect to the variational parameters does not depend on the partition function $Z$, i.e.
\begin{align*}
    \frac{d}{d\theta} \textrm{KL}(q_\theta, p)  
    &= \frac{d}{d\theta} \mathcal{F}(\theta)  \,,
\end{align*}
where we introduced
\begin{align*}
    \mathcal{F}(\theta) &= \mathbb{E}_{z \sim q_z} \left[ \ln q_\theta(g_\theta(z)) + E(g_\theta(z)) \right] \\
    & = \mathbb{E}_{z \sim q_z} \left[f(\theta, g_\theta(z)) \right] \,,
\end{align*}
with energy $E(x)=-\ln \hat{p}(x)$ and $f(\theta, g_\theta(z))=\ln q_\theta(g_\theta(z)) + E(g_\theta(z))$. One commonly refers to the quantity $\mathcal{F}$ as the variational free energy and to its negative $-\mathcal{F}$ as the evidence lower bound (ELBO). In the physics community, the former terminology is preferred while the machine learning literature tends to use the latter.

Using the reparameterization trick \eqref{eq:reparamTrick}, the total derivative of the reverse KL divergence can be written as
\begin{align}
    \tfrac{d}{d \theta} \textrm{KL}(q, p) &= \tfrac{d}{d \theta} \mathcal{F}(\theta) = \mathbb{E}_{z \sim q_Z} [\tfrac{d}{d \theta} f(\theta, g_\theta(z)) ] \,.     \label{eq:totalDer}
\end{align}
The total derivative of the function $f$ with respect to the variational parameters can be decomposed as follows
\begin{align*}
    \frac{d}{d\theta} f(g_\theta(z),\theta) = \blacktriangledown_\theta f(g_\theta(z),\theta) + \left. \frac{\partial}{\partial \theta} f(x, \theta) \right|_{x=g_\theta(z)} \,, 
\end{align*}
where we have introduced the \emph{path-gradient}
\begin{align*}
   \blacktriangledown_\theta f(g_\theta(z), \theta) = \frac{\partial f(g_\theta(z), \theta)}{\partial g_{\theta}(z)}^\intercal \frac{\partial g_{\theta}(z)}{\partial \theta} \,,
\end{align*}
i.e. the path-gradient only takes into account the implicit dependency on $\theta$ through the reparameterization $g_{\theta}$ and is insensitive to any explicit dependency.
The total derivative of the reverse KL divergences therefore splits in two terms
\begin{align*}
     \tfrac{d}{d \theta} &\textrm{KL}(q, p) \\ &= \mathbb{E}_{z \sim q_Z} [\blacktriangledown_\theta f(g_\theta(z),\theta) ] + \mathbb{E}_{z \sim q_Z} [ \left. \frac{\partial}{\partial \theta} f(x, \theta) \right|_{x=g_\theta(z)} ] \,.
\end{align*}
The latter is known as the score term and vanishes because the reparameterization trick \eqref{eq:reparamTrick} implies that
\begin{align*}
    &\mathbb{E}_{z \sim q_Z} [ \left. \frac{\partial}{\partial \theta} f(x, \theta) \right|_{x=g_\theta(z)} ] = \mathbb{E}_{z \sim q_Z} [ \left. \frac{\partial}{\partial \theta} \ln q_\theta(x) \right|_{x=g_\theta(z)} ] \\ 
    &= \mathbb{E}_{x \sim q_\theta} [ \frac{\partial}{\partial \theta} \ln q_\theta(x) ] = \frac{\partial}{\partial \theta} \int \textrm{d}x \, q_\theta(x) = 0  \,.
\end{align*}

The total derivative of the reverse KL divergence \eqref{eq:totalDer} is then estimated by Monte-Carlo
\begin{align}
    \frac{d}{d\theta} \textrm{KL}(q_\theta, p)  
    \approx \mathcal{G}_{\textrm{total}} = \mathcal{G}_{\textrm{score}} + \mathcal{G}_\textrm{path} 
\end{align}
where we have defined
\begin{align}
    \mathcal{G}_\textrm{path}  &= \frac{1}{N} \sum_{i=1}^N \blacktriangledown_\theta \left( \ln q_\theta(g_\theta(z_i)) + E(g_\theta(z_i)) \right) \,, \label{eq:pathgradest}\\
    \mathcal{G}_\textrm{score} &= \frac{1}{N} \sum_{i=1}^N \frac{\partial}{\partial \theta} \ln q_\theta(g_\theta(z_i)) \,,
\end{align}
with $z_i \sim q_Z$. We will refer to $\mathcal{G}_\textrm{path}$ as the path-gradient estimator and to $\mathcal{G}_\textrm{score}$ as the score estimator.

Both the total estimator $\mathcal{G}_{\textrm{total}}$ and the path-gradient estimator $\mathcal{G}_{\textrm{path}}$ are unbiased estimators of the gradient of the reverse KL divergence \eqref{eq:totalDer}.
 However, the estimators $\mathcal{G}_{\textrm{path}}$ and $\mathcal{G}_{\textrm{total}}$ have generically different variances because the variance of the score $\mathcal{G}_{\textrm{score}}$ is given by the Fisher information $\mathcal{I}(\theta)$ of the variational density $q_\theta$, i.e.
\begin{align*}
\mathrm{Cov}_{z_1, \dots, z_N \sim q_Z} \left[ \mathcal{G}_{\textrm{score}} \right]= \mathcal{I}(\theta) \,.
\end{align*}
We refer to the Supplement~\ref{app:dg} for a more detailed discussion. 

This difference in variance of the two estimators is particularly illustrative for a variational density $q_{\theta^*}$ which perfectly approximates the target $p$, i.e.
\begin{align}
    \forall x \in \mathcal{X}: && q_{\theta^*}(x) = p(x) \,. \label{eq:allequal}
\end{align}
For this case, the path-gradient estimator $\mathcal{G}_\textrm{path}$ vanishes identically because
\begin{align*}
    \mathcal{G}_\textrm{path} &= \frac{1}{N} \sum_{i=1}^N \left. \blacktriangledown_\theta \ln \frac{q_\theta(g_{\theta}(z_i))}{p(g_{\theta}(z_i))} \right|_{\theta^*} \\
    &= \frac{1}{N} \sum_{i=1}^N \underbrace{\frac{\partial}{\partial x_i} \ln  \frac{q_{\theta^*}(x_i)}{p(x_i)}}_{=0} \left. \, \frac{\partial g_\theta}{\partial \theta} \right|_{\theta^*} = 0 \,,
\end{align*}
where we have used \eqref{eq:allequal} in the last step.
As a result, the total gradient estimator $\mathcal{G}_{\textrm{total}}$ has non-vanishing variance $\mathcal{I}(\theta^*)$ even if the target $p$ is perfectly modelled by the variational density $q_{\theta^*}$. 
The path-gradient estimator $\mathcal{G}_{\textrm{path}}$ on the other hand does not have this problem. 

By continuity, one may therefore expect that the path-gradient estimator $\mathcal{G}_{\textrm{path}}$ has lower variance than the total gradient estimator $\mathcal{G}_{\textrm{total}}$ if the variational density $q_\theta$ is close to the target density $p$, i.e. in the final phase of training.

\begin{figure}[ht]
  \vskip 0.2in
  \begin{center}
  \centerline{\includegraphics[width=\columnwidth]{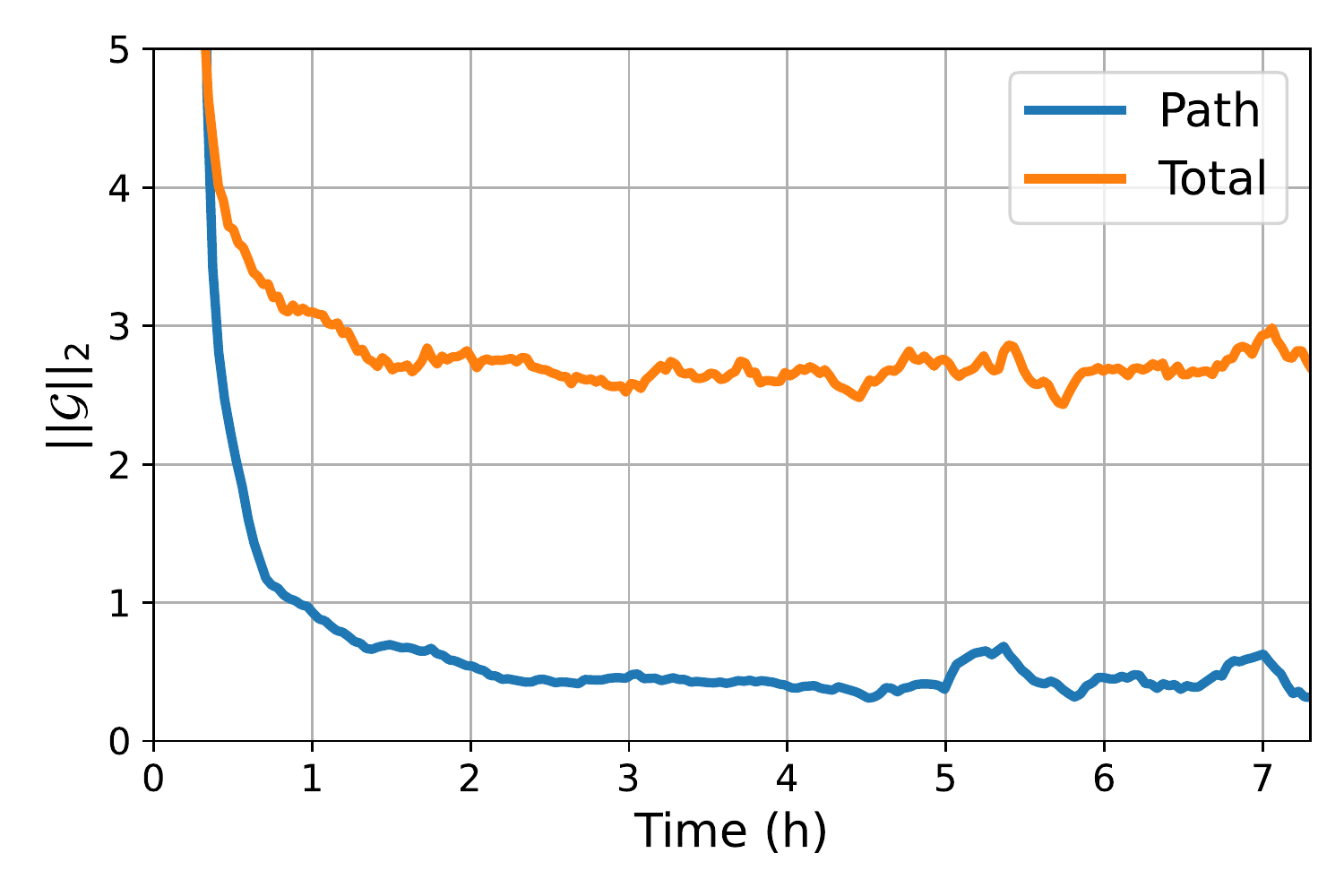}}
\caption{Norm of gradient estimators averaged over three runs for the $\phi^4$ lattice field theory with a rolling sum over 10 epochs for a lattice size of $L=12$.}
    \label{fig:gradnorm}
  \end{center}
  \vskip -0.2in
\end{figure}

\subsection{Path-Gradients for Simple Densities}\label{sec:pathgradsimple}
For the path-gradient estimator \eqref{eq:pathgradest}, we need to calculate
\begin{align*}
\blacktriangledown_\theta  \ln q_\theta(g_\theta(z)) + \blacktriangledown_\theta E(g_\theta(z)) \,.
\end{align*}
The second term can trivially be obtained by automatic differentiation because the total derivative
leads to the same result as the path gradient, i.e. $\frac{d}{d\theta} E(g_\theta(z)) = \blacktriangledown_\theta E(g_\theta(z))$. The first term is however non-trivial to calculate.

\citet{roeder2017sticking} proposed an algorithm for estimating path-gradients which proceeds in two steps: 
\begin{enumerate}
    \item Sample from $q_\theta$ by
    \begin{align*}
        x = g_\theta(z) \,.
    \end{align*}
    One can differentiate through this operation due to the reparameterization property of the density.
    \item Evaluate the probability density of the sample by formally using different parameters $\theta'$. The path-gradient is obtained by differentiating the result with respect to the original parameters $\theta$
    \begin{align*}
        \blacktriangledown_\theta \ln q(x, \theta) = \frac{d}{d \theta} \ln q(g_\theta(z), \theta') \,. 
    \end{align*}
\end{enumerate}
This approach can be efficiently implemented if the sampling process can be effectively disentangled from the evaluation of the log density, as is the case for (mixtures of) Gaussians and for other density known in closed-form. 

For normalizing flows, this is however not the case: the density contains the determinant of the Jacobian $|\frac{d g_\theta}{dz}|$ which is calculated during the sampling process. To the best of our knowledge, the only reference using path-gradients for invertible normalizing flow is \citet{agrawal2020advances}. In this reference, the authors calculate the path-gradient for a normalizing flow by considering two separate copies of the model for sampling and density evaluation. This doubles the runtime as well as the memory footprint and thereby reduces the possible mini-batch sizes in training. Large mini-batch sizes are however of central importance for successful training in many relevant application domains (see e.g. \citet{del2021machine,  kanwar2021machine}). 

For the simple class of variational densities discussed above, \citet{roeder2017sticking, tucker2018doubly, finke2019importanceweighted, geffner2020difficulty} have convincingly established that path-gradients lead to better optimization behaviour in practice.

\section{Path-Gradients for Continuous Normalizing Flows}\label{sec:pathgradcnf}
In the following, we will introduce an efficient algorithm to calculate path-gradients for continuous normalizing flows. 

Continuous normalizing flows \citep{chen2018neural} transform a sample $z_0 \sim q_Z$ from a base-density $q_Z$ using 
\begin{align}
    x \equiv z_T &= g_\theta(z_0) \nonumber \\
    &= z_0 + \int_0^T \, \textrm{d}t \, f_\theta(z_t, t)\,. \label{eq:cnfoutput}
\end{align}
Under the mild assumptions of the Picard–Lindelöf theorem, the map $g_\theta$ is bijective. In practice, the map $g_\theta$ is implemented using a Neural Ordinary Differential Equation (NODE) \citep{chen2018neural} of the form
\begin{align}
    \frac{d z_t}{d t} &= f_\theta(z_t, t) \label{eq:node}\,,
\end{align}
where $f_\theta$ is a (not necessarily invertible) neural network with parameters $\theta$ and $z_t$ is the intermediate state at time $t$. The determinant of the Jacobian of the reparameterization $g_\theta$ is given by
\begin{align*}
 \ln \left| \det 	\frac{ d g_\theta}{d z_0}  \right| = \int_0^T \textrm{tr} \left(\frac{d f_\theta(z_t, t)}{d z_t} \right) \textrm{d}t \,.
\end{align*}
The logarithm of the variational density is thus given by
\begin{align*}
    \ln q_\theta&(x) = \ln q_Z (g_\theta^{-1}(x))  -  \ln\left| \det  	\frac{ d g_\theta}{d z_0} \right|_{z_0 = g^{-1}_\theta(x)} \\
        & = \ln q_Z (z_0) - \int_0^T \textrm{tr} \left( \frac{\partial f_\theta(z_t, t)}{\partial z_t} \right) \textrm{d}t \,.
\end{align*}
In practice, the trace is often estimated using the Hutchinson trace estimator \citep{grathwohl2018scalable}.

\begin{figure*}
\vskip 0.2in
\begin{center}
\includegraphics[width=\textwidth]{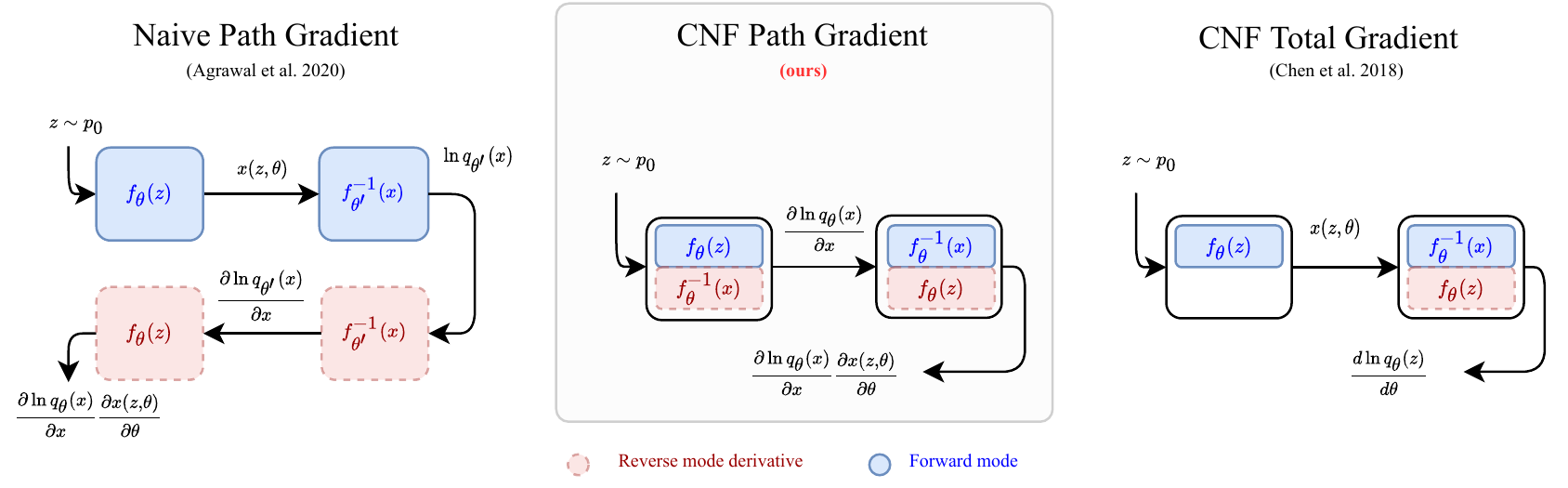}
\caption{\textbf{Left:}  method of \citet{agrawal2020advances}. Top row: two cloned models $q_\theta$ and $q_{\theta'}$ are used for sampling and density evaluation. This doubles the memory requirements. Lower row: reverse-mode differentiation through these two cloned models. Each NODE backward pass has the cost of two forward passes. Total cost of \cite{agrawal2020advances}: about six forward passes; double memory due to the cloning of the models. \textbf{Middle:} our method. Forward pass to calculate $x=z_T$ and the derivative $\tfrac{\partial \ln q_\theta(z_T)}{\partial z_T}$ at cost of about two forward passes and constant memory. By reverse-mode differentiation, the vector Jacobian product \eqref{eq:scoreVJP} is obtained at costs of about two forward passes and constant memory requirements. Total cost of our method: about four forward passes; half the memory requirements compared to \citet{agrawal2020advances}. \textbf{Right:} standard total gradient approach. Forward is followed by backward pass for which intermediate states are recomputed for memory efficiency. Total costs: about three forward passes. Our method has the same memory costs and $4/3$ times the computational costs but nevertheless converges more quickly.}
\label{fig:graph-pseudocode}
\end{center}
\vskip -0.2in
\end{figure*}

\subsection{Total Derivatives}\label{sec:totalder}
One major insight of \citet{chen2018neural} was the fact that we can use the adjoint state method~\citep{pontryagin1987mathematical} for computing the gradients of NODE model instead of using standard backpropagation through the ODE solver. Specifically, for any loss function of the form
\begin{align*}
    L(x) = L(z_T) = L\left(z_0 + \int_0^T \, \textrm{d}t \, f_\theta(z_t, t) \right) \,,
\end{align*}
its derivative with respect to the parameters $\theta$ can be obtained by
\begin{align*}
    \frac{dL}{d\theta} = - \int_0^T \textrm{d}t \, \left( \frac{\partial L}{ \partial z_t} \right)^\intercal \, \frac{\partial f_\theta(z_t, t)}{\partial \theta} \,,
\end{align*}
where the adjoint state $a_t = \tfrac{dL}{dz_t}$ is obtained by solving the terminal value problem
\begin{align}
    &\frac{d a_t}{dt} = - a_t^\intercal \frac{\partial f_\theta(z_t, t)}{\partial z_t}  \,, \nonumber \\
    &a_T = \frac{dL}{dz_T} 
    \,.  \label{eq:adjstate}
\end{align}
This approach allows to compute the gradients of the NODE with constant memory requirements at the price of additional computational costs. This proceeds in two steps:
\begin{enumerate}
    \item Calculation of $z_T$ by forward integration of \eqref{eq:node}. The intermediate states $z_t$ are \emph{not} stored in this forward pass to ensure constant memory costs.
    \item The adjoint terminal value problem \eqref{eq:adjstate} is then integrated starting from the terminal condition $a_T=\tfrac{dL}{dz_T}$ to obtain the adjoint states $a_t=\tfrac{dL}{dz_t}$. Since these adjoint states have the same dimensionality as the states $z_t$, this can be done for about the same computational cost as the forward pass. However, we also need to recompute the intermediate states $z_t$ by an additional reverse evolution of \eqref{eq:node} as they were not stored in the forward pass but enter the adjoint terminal value problem \eqref{eq:adjstate}. So in total, solving \eqref{eq:adjstate} has the computational cost of about two forward passes for constant memory requirements.
\end{enumerate}
 In summary, the total derivative $\frac{dL}{d\theta}$ of a loss $L$ can be calculated at constant memory with about the computational cost of three forward passes.

\subsection{Path Gradients}
We now discuss how to calculate the path-gradient for continuous normalizing flows. We recall from our discussion in Section~\ref{sec:pathgradsimple} that the only non-trivial term of the loss involves
\begin{align}
    \blacktriangledown_\theta  \ln q_\theta(g_\theta(z_0)) &= \frac{\partial \ln q_\theta(g_\theta(z_0))}{\partial g_\theta(z_0)}^\intercal \frac{\partial g_\theta(z_0)}{\partial \theta}  \nonumber \\
    &=\frac{\partial \ln q_\theta(z_T)}{\partial z_T}^\intercal \frac{\partial z_T}{\partial \theta} \label{eq:scoreVJP} \,,
\end{align}
where we have used the relation $z_T = g_\theta(z_0)$ as defined in \eqref{eq:cnfoutput}.

We prove in Appendix~\ref{app:CNF-adjoint} that the derivative of the log density $\frac{\partial \ln  q_\theta(z_T)}{\partial z_T} $ with respect to the final state $z_T$ can be calculated using the following ODE:

\begin{theorem} \label{th:score}
The derivative $\tfrac{\partial \ln q_\theta(z_T)}{\partial z_T}$ can be obtained by solving the initial value problem
\begin{align}
    \frac{d}{dt}   \frac{\partial \ln  q_\theta(z_t)}{\partial z_t} = 	&-\frac{\partial \ln q_\theta(z_t)^\intercal}{\partial z_t}  \frac{\partial f_\theta(z_t,t)}{\partial z_t} \nonumber \\ & \hspace{1em} - 	 \partial_{z_t} \textrm{tr}\left(\frac{\partial f_\theta(z_t, t)}{\partial z_t} \right) \,, \label{eq:pathgradode}
\end{align}
with initial condition
\begin{align*}
    \frac{\partial \ln q_\theta(z_0)}{\partial z_0} =  \frac{\partial \ln q_Z(z_0)}{\partial z_0} \,.
\end{align*}
\end{theorem}

We note that the derivative $\frac{\partial \ln  q_\theta(z_t)}{\partial z_t}$ is of the same dimension as the intermediate state $z_t$. Thus, solving its initial value problem of Theorem~\ref{th:score} has about the same computational cost as the forward pass which calculates $z_T$. As a result, the joint calculation of $\frac{\partial \ln  q_\theta(z_t)}{\partial z_t}$ and $z_T$ has about the cost of two forward passes. We then need to calculate the vector Jacobian product \eqref{eq:scoreVJP}. This can be done with reverse-mode differentiation using the adjoint state method discussed in Section~\ref{sec:totalder} for the costs of about two forward passes (one for the evolution of the adjoint state and the other to re-compute the intermediate states $z_t$). 

In summary, our proposed method to calculate the path-gradient for continuous normalizing flows has the computational cost of about four forward passes as compared to three for the total derivative. Crucially, it also has constant memory requirements.

\begin{figure*}
\vskip 0.2in
\begin{center}
\includegraphics[width=0.33 \textwidth]{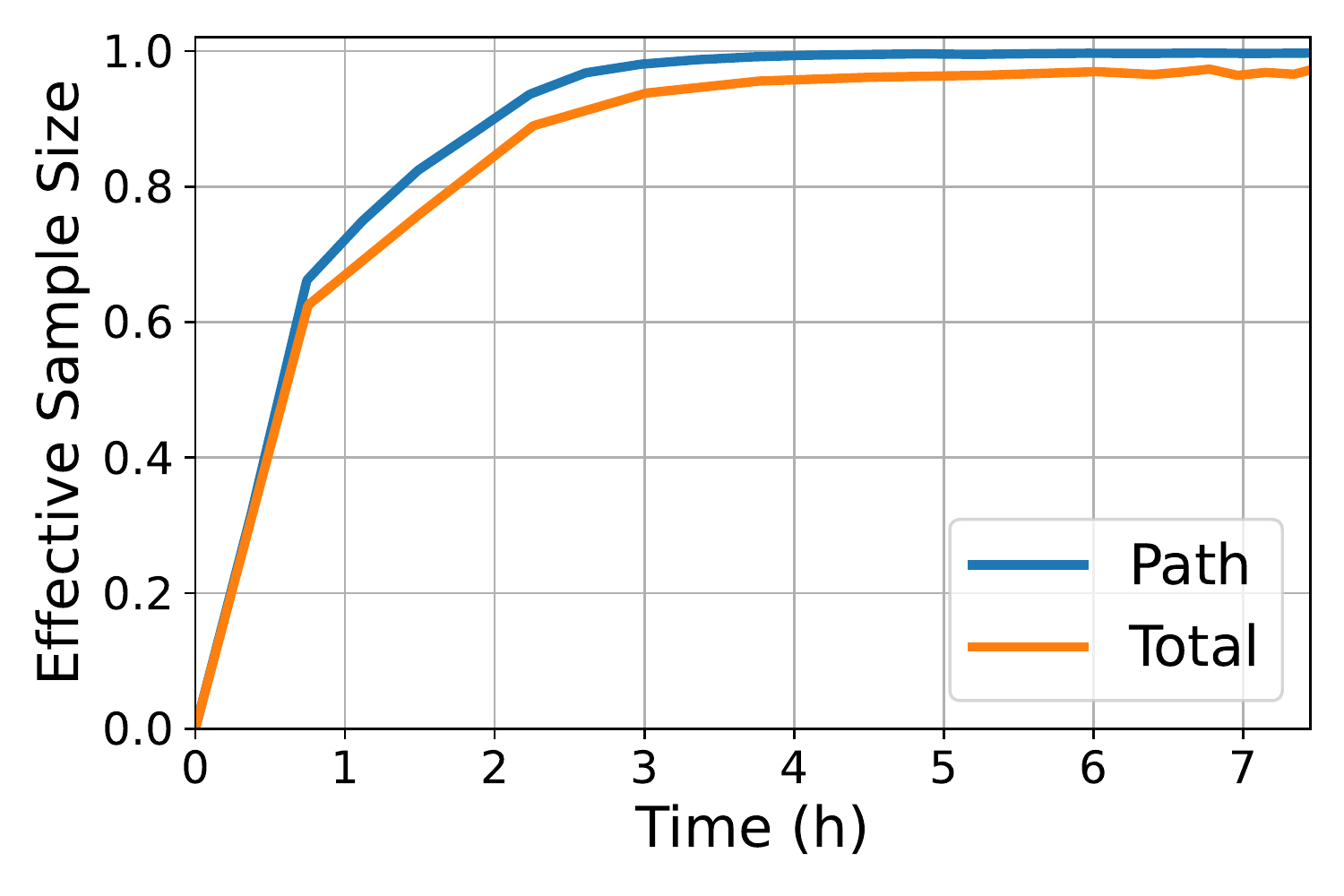}
\includegraphics[width=0.33 \textwidth]{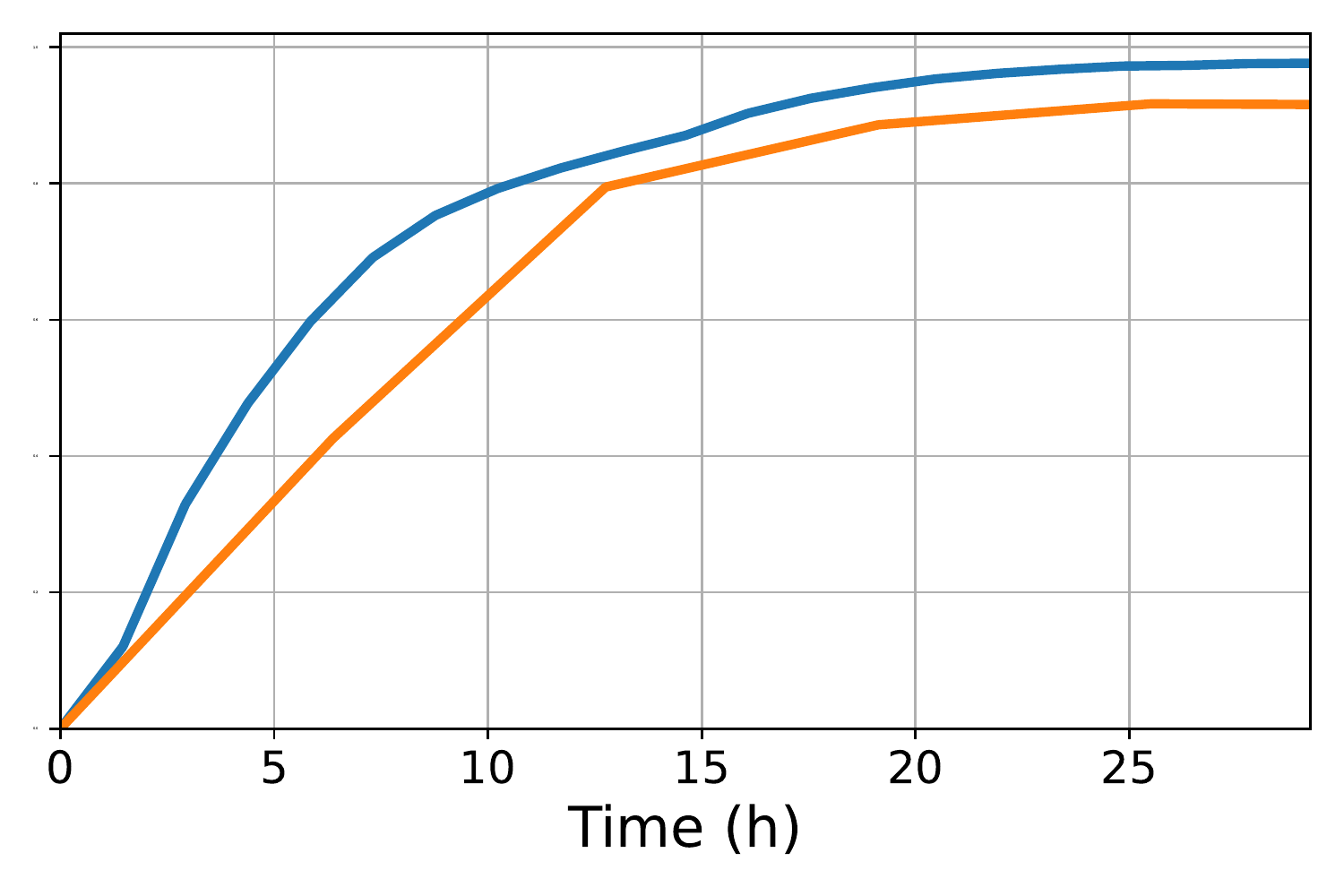} 
\includegraphics[width=0.33 \textwidth]{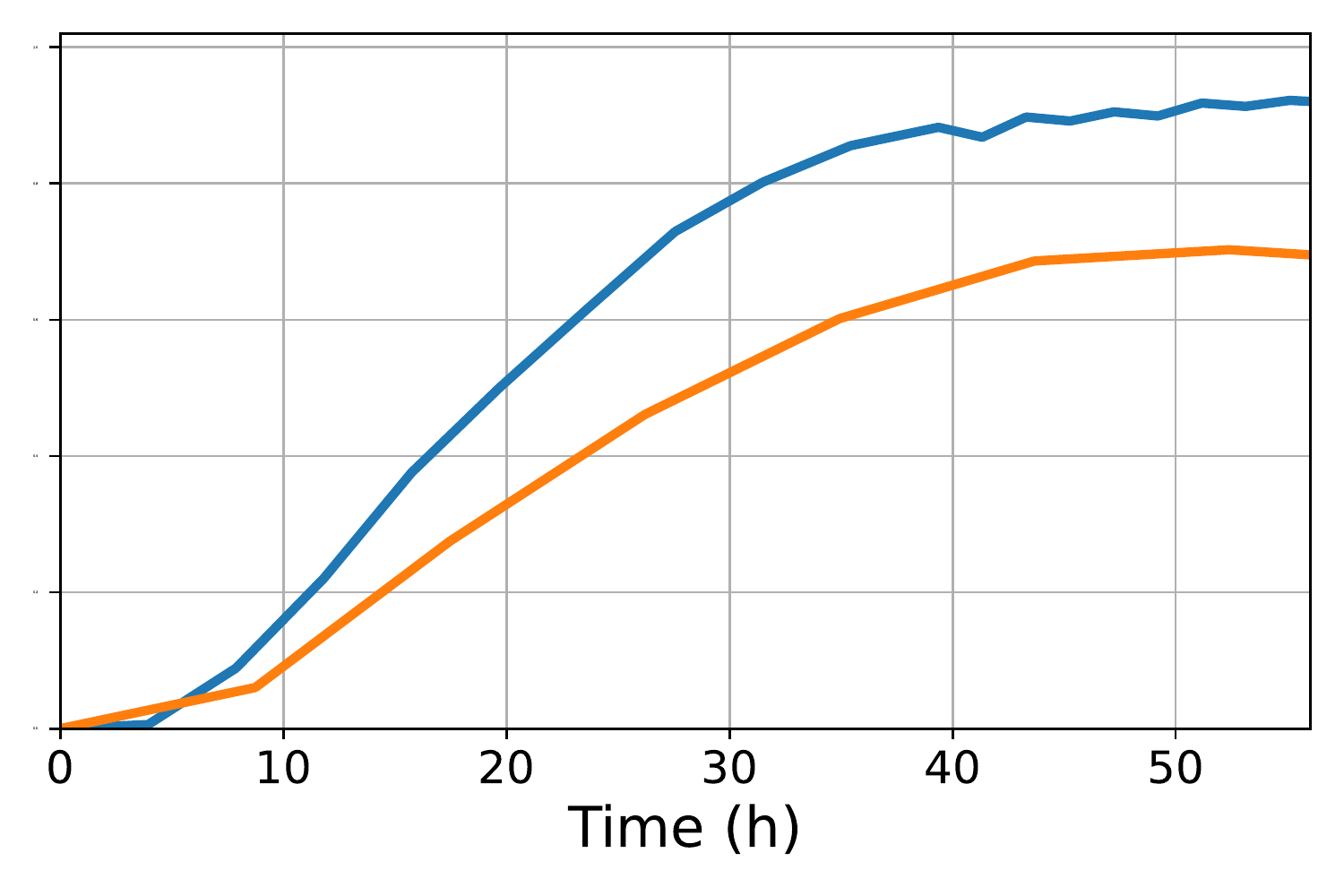}
\caption{Evolution of effective sampling size (ESS) during training; \textbf{Larger is better}. Training for $L$x$L$-lattices with $L=12$ (\textbf{left}) and $L=20$ (\textbf{middle}), as well as $L=32$ (\textbf{right)}. The path-gradient estimator leads to faster convergence and overall better ESS values. Training was performed on a single $A100$ GPU.}
\label{fig:phi4-training}
\end{center}
\vskip -0.2in
\end{figure*}

The following comments are in order:
\begin{itemize}
    \item The initial value problem \eqref{eq:pathgradode} of Theorem~\ref{th:score} is solved forward in time. As such the intermediate states $z_t$ do not need to be re-computed for this in order to satisfy the constant memory constraint. This is in contrast to the reverse-mode differentiation with the adjoint state method discussed in Section~\ref{sec:totalder}.
    \item Our estimate of the runtime ratio of about $4/3$ for the path-gradient compared to the total derivative does not take account synergies between the various ODEs. For example, the derivative $\frac{\partial f_\theta(z_t,t)}{\partial z_t}$ appears both in the adjoint terminal value problem \eqref{eq:adjstate} and in the initial value problem \eqref{eq:pathgradode} of Theorem~\ref{th:score}. This will result in an effectively smaller difference in runtime as we will demonstrate in our numerical experiments in Section~\ref{sec:runtime}. 
    \item In our numerical experiments, we will demonstrate that this runtime overhead of the path-gradient estimator will be more than compensated by better convergence properties of the path-gradient estimator, see Section~\ref{sec:experiments}. This may be expected due to the smaller variance of the path-gradient estimator, as discussed in Section~\ref{sec:vi}.
    \item Our method compares very favorably to the approach taken in \citet{agrawal2020advances} which has the cost of about six forward passes and twice the memory footprint (compared to both the standard total- and our path-gradient estimator). See Figure~\ref{fig:graph-pseudocode} for an illustration. 
\end{itemize}

\section{Numerical Experiments}\label{sec:experiments}
We consider two different application domains, i.e. Variational Autoencoders and normalizing flows for lattice field theories. We demonstrate that a simple replacement of the total derivative estimator by the path-gradient estimator improves the performance on both of these tasks. We also discuss the runtime costs per iteration and for the entire training process. We provide code to reproduce the experiments using VAEs\footnote{ \href{https://github.com/lenz3000/ffjord-path}{https://github.com/lenz3000/ffjord-path}}.

\subsection{Lattice Field Theory}
Recently, there has been substantial interest in applying normalizing flows to Lattice Field Theories which can be used to describe the strong interactions of elementary particles such as electrons and quarks. From the machine learning perspective, this application is exciting as incorporating the symmetries of the field theories in the model is of the utmost importance in this domain. Continuous normalizing flows are particularly suitable for ensuring such an inductive bias (see \citet{kohler2020equivariant}). 

We repeat the state-of-the-art experiments in \citet{haan2021scaling} for a two-dimensional scalar lattice field theory with a quartic interaction. 
\paragraph{$\phi^4$-theory:} this lattice field theory is described by a random vector $\phi$ with components $\phi_x$ where the index $x \in \Lambda$ takes values on an $L \times L$ square lattice $\Lambda$. The random vector $\phi$ has the density $p(\phi)=\tfrac{1}{Z} \exp(-S(\phi))$ with action
\begin{align*}
    S(\phi) = \sum_{x,y \in \Lambda} \phi_x \bigtriangleup_{xy} \phi_y + \sum_{x \in \Lambda} m^2 \, \phi_x^2 + \lambda \, \phi_x^4 \,,
\end{align*}
where $\bigtriangleup_{xy}$ is the Laplacian matrix on the lattice $\Lambda$ and we use periodic boundary conditions. The density $p$ has a $\mathbb{Z}_2$-symmetry, i.e. $p(\phi)=p(-\phi)$. This density is also invariant under the spatial symmetry group $G = C_L^2 \rtimes D_4 $ of the lattice $\Lambda$, i.e. the semi-direct product of the squared cyclic group $C_L$ and the Dihedral group $D_4$. We use the state-of-the-art model of \citet{haan2021scaling} which is manifestly invariant under these symmetries.  We also use the same choices for bare mass $m^2$ and coupling $\lambda$ as in \citet{haan2021scaling} as well as their three largest lattices.
\paragraph{Training:} the reverse KL divergence $\textrm{KL}(q_\theta, p)$ from the model $q_\theta$ to the target density $p$ is minimized using either the standard total derivative estimator $\mathcal{G}_{\textrm{total}}$ or the path-gradient estimator $\mathcal{G}_{\textrm{path}}$. Then the performance of the respectively trained models is compared.
Both the path-gradient and total gradient training use the same training hyperparameters as in \citet{haan2021scaling} except for a larger batch-size of 5000 and learning rate which we find to improve the performance of the baseline. Each model was trained on a single A100 GPU. We refer to the Supplement~\ref{app:exp-details} for more details on the training procedure.
\paragraph{Effective Sampling Size:} one can then estimate expectation values with respect to the \emph{target density} $p$, i.e. physical observables, using self-normalized importance sampling \citep{muller2019neural, noe2019boltzmann, nicoli2020asymptotically}:
\begin{align*}
    \mathbb{E}_{\phi \sim p} \left[ F(\phi) \right] \approx \hat{F} \equiv \sum_{i=1}^N \frac{\tilde{w}_i}{\sum_{j=1}^N \tilde{w}_j} \, F(\phi_i) \,, && \phi_i \sim p \,,
\end{align*}
where $\tilde{w}_i = \tfrac{\exp(-S(\phi_i))}{q(\phi_i)}$ and $F$ is an arbitrary function. For a variational density $q_\theta$ which perfectly approximates the target density $p$, the variance of the estimator $\hat{F}$ scales as $1/N$. Due to the approximation error of $q_\theta$, in practice a scaling with $1/N_{\textrm{eff}}$ is observed with $N_{\textrm{eff}} \le N$. The quality of the approximation of the target density $p$ by the variational density $q_\theta$ is therefore commonly measured in terms of the effective sampling size
\begin{align}
    \textrm{ESS} = \frac{N_{\textrm{eff}}}{N} \in [0, 1] \,,
\end{align}
where values close to the limits of 1 and 0 indicate perfect or very poor approximation of the target, respectively. 

In practice, the ESS can be estimated by
\begin{align*}
    \textrm{ESS} \approx \frac{(\frac{1}{N} \sum_{i=1}^N \tilde{w}_i)^2}{\frac{1}{N} \sum_{i=1}^N \tilde{w}_i^2} \,.
\end{align*}
We refer to \citet{nicoli2020asymptotically} for more details.
\paragraph{Results:} a simple replacement of the total derivative estimator $\mathcal{G}_{\textrm{total}}$ by the path-gradient estimator $\mathcal{G}_{\textrm{path}}$ leads to a significant faster learning process as demonstrated in Figure~\ref{fig:phi4-training}. We stress that this speed-up is in terms of runtime and not only number of iterations. Furthermore, Table~\ref{tab:phi4} demonstrates that the final effective sampling size of the models trained by path-gradients is consistently larger as compared to models trained by the total derivative estimator. This effect is more pronounced as the lattice size increases.

 Figure~\ref{fig:gradnorm} shows that the norm of the path-gradient estimator is significantly lower than the norm of the standard estimator, as expected from the discussion in Section~\ref{sec:vi}.

\begin{table}[t]
  \caption{Effective Sample Size estimated over 500k samples after 24k epochs for total derivative estimator $\mathcal{G}_{\textrm{total}}$ and 20k for path-gradient estimator $\mathcal{G}_{\textrm{path}}$; {\bf larger is better}.
  $^\dagger$ For $L=32$, we trained for 15.4k/13k epochs.}
  \label{tab:phi4}
\vskip 0.15in
\begin{center}
\begin{small}
\begin{sc}
    \begin{tabular}{@{}l|rr}
  Lattice size & Path  & Total\\
  \toprule
  12x12  & $\textbf{99.66} \pm 0.07$ & $98.01 \pm 0.44 $ \\
  20x20  & $\textbf{97.65} \pm 0.14$ & $91.56 \pm 1.13$ \\
  32x32$^\dagger$  & $\textbf{91.81} \pm 1.32$ & $69.53 \pm 5.59$  \\
  \end{tabular}
\end{sc}
\end{small}
\end{center}
\vskip -0.1in
\end{table}

\subsection{Variational Autoencoder}

We repeat the VAE experiments in \citet{grathwohl2018scalable} which train a VAE for four datasets using a FFJORD flow. 
We use the path-gradient $\mathcal{G}_{\textrm{path}}$ and the total derivative estimator $\mathcal{G}_{\textrm{total}}$ for training and compare their respective performances. 
In contrast to the lattice field theory application above, these experiments use early stopping and an adaptive step size ODE solver. We refer to the Supplement~\ref{app:exp-details} for more details. 

Table~\ref{tab:inference} summarizes the results for the ELBO. On all considered datasets, the path-gradient estimator outperforms the total gradient estimator. 

 \begin{table}[t]
 \caption{Negative ELBO on test data for VAE models using FFJORD flows; \textbf{lower is better}. In nats for all datasets except Frey Faces which is presented in bits per dimension. Mean/stdev are estimated over 3 runs. Baseline total derivative estimator results are as in \citet{grathwohl2018scalable}.}
  \label{tab:inference}
\vskip 0.15in
\begin{center}
\begin{small}
\begin{sc}
  \begin{tabular}{@{}l|rr}
   Dataset  & Path & Total\\
  \toprule
  MNIST  & $\textbf{82.09} \pm .04$ & $82.82 \pm .01 $\\
  Omniglot  & $\textbf{96.61} \pm .17$ & $ 98.33 \pm .09 $\\
  Caltech Silhouettes  &  $\textbf{101.93} \pm .63$ & $ 104.03 \pm .43 $\\
  Frey Faces  & $\textbf{4.35} \pm .00$ & $4.39 \pm .01$\\
  \end{tabular}
\end{sc}
\end{small}
\end{center}
\vskip -0.1in
\end{table}

 \subsection{Runtime Comparison} \label{sec:runtime}
 We empirically evaluate the runtime costs of our path-gradient estimator compared to the standard total gradient estimator on the two considered tasks.

 \paragraph{Lattice Field Theory:}
 we stress that we performed both the path-gradient and total gradient training for the same overall runtime to ensure fair comparison. As discussed in Section~\ref{sec:pathgradcnf}, we expect the path-gradient estimator to have slightly higher costs \emph{per iteration}. We find that this effect is more than compensated by better convergence properties of the path-gradient. 
 
 As expected, we furthermore find in Table~\ref{tab:phi4time} that the runtime overhead per iteration is between $14\%$ and $5\%$ depending on the lattice size and thereby substantially lower than the estimate of $33\%$ which neglected synergies between the various ODEs, as discussed in Section~\ref{sec:pathgradcnf}. 
 
 The difference in iteration costs decreases as the size of the lattice increases. This effect can be attributed to fact that the continuous normalizing flow architecture of \citet{haan2021scaling} has a quadratic scaling of model parameters with the lattice size. Even if we use a best-case scenario for the total gradient baseline, i.e. a model whose number of parameters is independent of the lattice size, this overhead is mild and stays constant with the problem size, see Figure~\ref{fig:runtime}. Furthermore, our method is significantly faster than the method to estimate the path-gradient proposed by \citet{agrawal2020advances} which furthermore has twice the memory costs.
 
 In summary, we find empirically that the overhead of our path-gradient estimation method per iteration is mild, decreasing as the dimension increases for the relevant architecture, and significantly outperforms the proposal by \citet{agrawal2020advances}. Crucially, the path-gradient leads to faster convergence and thus achieves better results using the same walltime.

\begin{figure}[t]
\vskip 0.2in
\begin{center}
\includegraphics[width=\columnwidth]{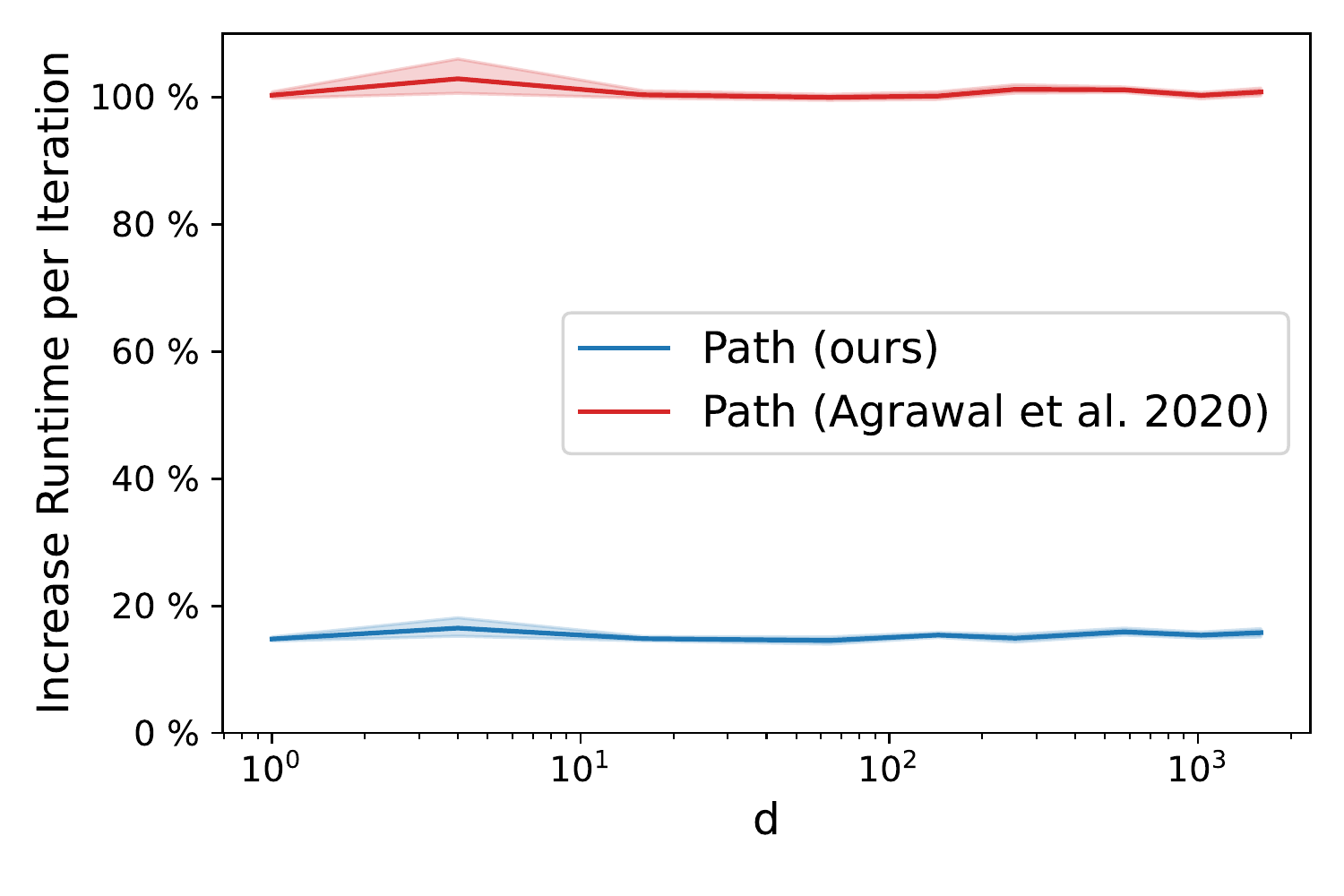}
\caption{Increase in runtime per iteration for computing the path gradient estimated by our method and the method by \citet{agrawal2020advances} as compared to the standard reparametrized gradient estimator measured on an A100 GPU.}
\label{fig:runtime}
\end{center}
\vskip -0.2in
\end{figure}

\paragraph{VAE:}
The VAE experiments use an adaptive step-size solver for evolving the ODEs. As such, the time per iteration cannot be meaningfully compared. The training also used early-stopping which necessarily implies that training has no fixed runtime. 

Nevertheless, we find that the path-gradient training has comparable runtime to the standard total gradient training (largest runtime of our method were $111\%$ for the Omniglot and $98\%$ for the MNIST dataset compared to standard estimator). 
Reassuringly, we did not see any noticeable effect of the path gradient estimator on the number of function evaluations (see Figure~\ref{fig:NFE}).
 
 \begin{figure}[t]
\vskip 0.2in
\begin{center}
\includegraphics[width=\columnwidth]{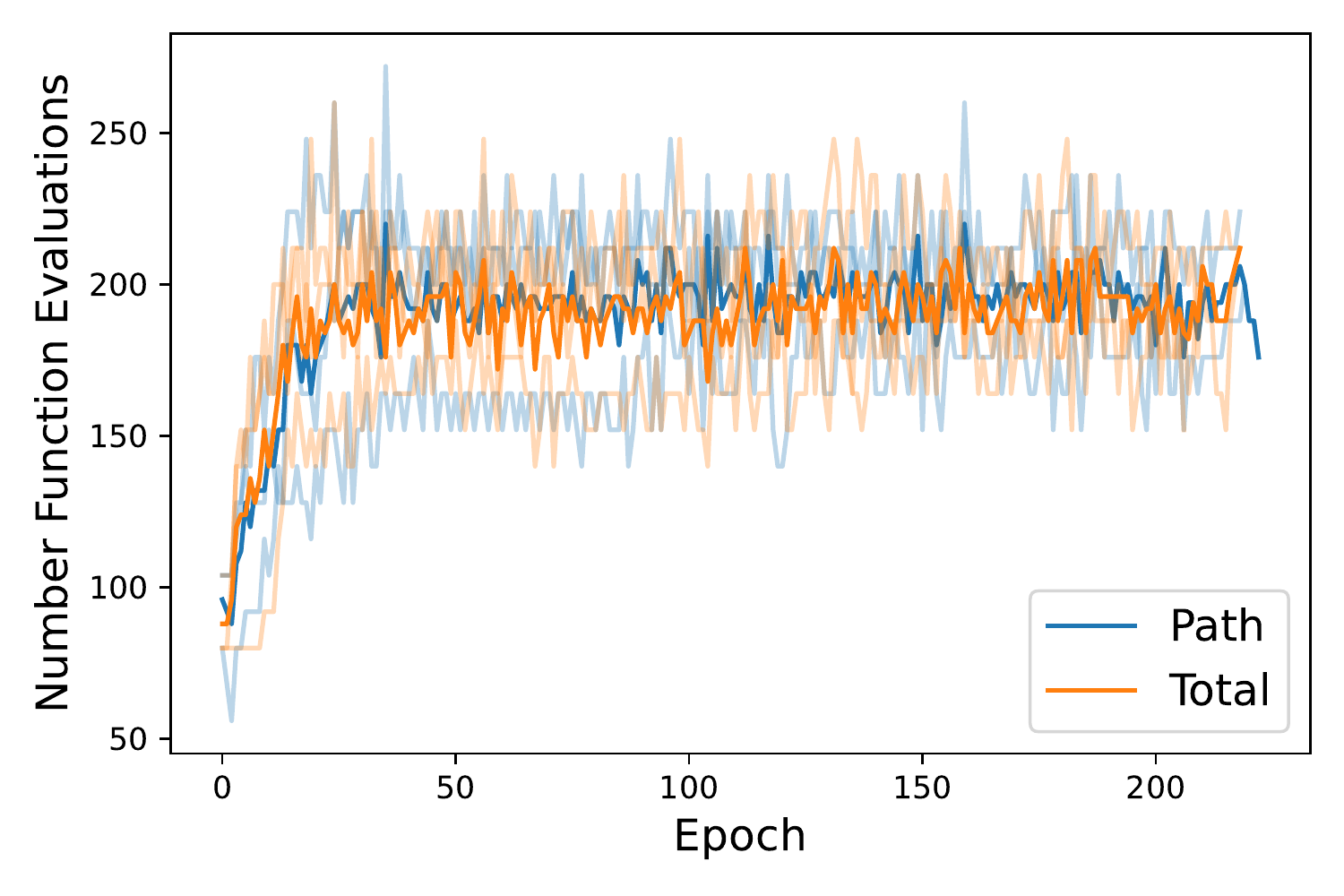}
\caption{Number of function evaluations for both gradient estimators on the Frey Faces dataset using the same hyperparameters. No trend is recognizable}
\label{fig:NFE}
\end{center}
\vskip -0.2in
\end{figure}

\begin{table}[t]
  \caption{Time increase \emph{per iteration} of the path-gradient over the standard total derivative estimator on an A100 GPU for the lattice field theory experiments of Table~\ref{tab:phi4}.}
  \label{tab:phi4time}
\vskip 0.15in
\begin{center}
\begin{small}
\begin{sc}
    \begin{tabular}{@{}l|r}
  Lattice size  & Time per Iteration\\
  \toprule
  12x12 &  $+14 \% \pm 0 $\\
  20x20 &  $+ 12 \% \pm 1$\\
  32x32 & $ + 5\% \pm 1$ \\
  \end{tabular}
\end{sc}
\end{small}
\end{center}
\vskip -0.1in
\end{table}
\section{Conclusion}
We have proposed a path-gradient estimator which can be efficiently implemented for continuous normalizing flows. 

For this, we derived an ODE \eqref{eq:pathgradode} which can be solved forwards in time to obtain the gradient of the variational log density with respect to the output of the flow. From this, we can then straightforwardly obtain the path-gradient.

This approach has same constant memory requirements and only slightly increased runtime per iteration as compared to the standard total gradient estimator. It also is significantly faster per iteration and has half the memory requirements compared to the method proposed by \citet{agrawal2020advances}. In particular, our method allows for training with sufficiently large mini-batches for modern Variational Inference tasks in the natural sciences, as demonstrated by the lattice field theory application considered in our experiments.

We show empirically that a straightforward plug-and-play replacement of the total gradient by the path-gradient estimator leads to a surprisingly pronounced increase in performance in both the VAE and lattice field theory tasks. As an example, we observe a larger than $30\%$ increase in ESS for highest dimensional lattice field theory task by switching from the standard total derivative to the path-gradient estimator.

A limitation of our work is that so far there is only a theoretical explanation for the lower variance of the path-gradient estimator in the final training phase, as was discussed in Section~\ref{sec:vi}. However, our experimental results suggest that the entire training process benefits from the path-gradient estimator. A theoretical understanding of this would be desirable but seems mathematically very challenging. 

In light of our results, we expect that the proposed path-gradient estimator will become a standard item in the toolbox for training continuous normalizing flows on modern Variational Inference tasks.

\section*{Acknowledgments}
We thank the anonymous reviewers for their valuable feedback.
K.A.N., S.N. and P.K. are supported by the German Ministry for Education and Research (BMBF) as BIFOLD - Berlin Institute for the Foundations of Learning and Data under grants 01IS18025A and  01IS18037A.





\bibliography{main}
\bibliographystyle{icml2022}

\newpage
\appendix
\onecolumn
\section{Path-Gradients and Information Geometry}\label{app:dg}
The Fisher information is defined by
\begin{align*}
    \mathcal{I}(\theta)_{ij} = \mathbb{E}_{x \sim q_\theta} \left[ s(\theta, x)_i \, s(\theta, x)_j \right) ] \,,
\end{align*}
where the score function is defined by
\begin{align*}
    s(\theta, x)_i = \frac{\partial}{\partial \theta_i} \ln q_\theta(x) \,.
\end{align*}

Using
\begin{align*}
 \mathbb{E}_{z_1, \dots, z_N \sim q_Z} \left[ \mathcal{G}_{\textrm{score}} \right] = 0 \,,
\end{align*}
the covariance of the score term $\mathcal{G}_{\textrm{score}}$ is given by
\begin{align*}
\mathrm{Cov}_{z^{(1)}, \dots, z^{(N)} \sim q_Z} \left[ \mathcal{G}_{\textrm{score}} \right]_{ij}
= \mathbb{E}_{z^{(1)}, \dots, z^{(N)} \sim q_Z} \left[ \frac{1}{N} \sum_{l=1}^N \partial_{\theta_i} \ln q_\theta(x(z^{(l)})) \frac{1}{N} \, \sum_{k=1}^N \partial_{\theta_j} \ln q_\theta(x(z^{(k)})) \right]
\end{align*}
We now use the fact that the samples $z_i$ are independent and identically distributed, e.g. $\mathbb{E}_{z^{(i)}, z^{(j)} \sim q_Z} [\bullet] = \mathbb{E}_{z^{(i)} \sim q_Z} [\bullet] \mathbb{E}_{z^{(j)} \sim q_Z} [\bullet]$ for $i \neq j$, to arrive at the final result
\begin{align*}
\mathrm{Cov}_{z^{(1)}, \dots, z^{(N)} \sim q_Z} \left[ \mathcal{G}_{\textrm{score}} \right]_{ij}
&=  \frac{1}{N} \sum_{l=1}^N \mathbb{E}_{z^{(l)} \sim q_Z }\left[ \partial_{\theta_i} \ln q_\theta(x(z^{(l)})) \,  \partial_{\theta_j} \ln q_\theta(x(z^{(l)})) 
\right] \\
&=\mathbb{E}_{z \sim q_Z }\left[ \partial_{\theta_i} \ln q_\theta(x(z)) \,  \partial_{\theta_j} \ln q_\theta(x(z))
\right] \\
&=\mathbb{E}_{x \sim q_\theta }\left[ \partial_{\theta_i} \ln q_\theta(x) \,  \partial_{\theta_j} \ln q_\theta(x)
\right] \\
&= \mathbb{E}_{x \sim q_\theta} \left[ s(\theta, x)_i \, s(\theta, x)_j \right) ] \\
&= \mathcal{I}(\theta)_{ij} \,. 
\end{align*}

\section{Algorithms}
\begin{algorithm*}
\caption{\textbf{Forw-Aug}: Forward-mode derivative for path-wise gradient estimators for CNFs}
\label{algo1}
\begin{algorithmic}
\STATE {\bfseries Input:} dynamics parameters $\theta$, start time $0$, stop time $1$, initial state $z_0$, 
loss gradient $\nicefrac{\partial \ln q_0(z_0)}{\partial z_0}$
\STATE $s_0 = [z_0, \frac{\partial \ln q_0(z_0)}{\partial z_0}]$ \COMMENT{Define initial augmented state}
\STATE{\textbf{def }\textnormal{aug\_dynamics}}{$([z_t, \alpha(t), \cdot], t, \theta)$}: \COMMENT{Define dynamics on augmented state}

\STATE \hspace{2em} \textbf{return} $[f(z_t, t, \theta), -\alpha(t)  \frac{\partial f}{\partial z_t}  - \frac{\partial \textrm{tr} \left(\nicefrac{df}{d z_t}\right)}{\partial z_t}]$ \COMMENT{Compute vector-Jacobian products} 

\STATE $[z_1, \ln q_\theta(z_T), \frac{\partial \ln q_\theta(z_T)}{\partial z_T}] = \textnormal{ODESolve}(s_0, \textnormal{aug\_dynamics}, 0, T, \theta)$ \COMMENT{Solve reverse-time ODE} 
\STATE {\bfseries Return:} $z_T, \ln q_\theta(z_T), \frac{\partial \ln q_\theta(z_T)}{\partial z_T}$ \COMMENT{Return gradients}
\end{algorithmic}   
\end{algorithm*}

\begin{algorithm*}
\caption{Full path gradient computation}
\label{algo2}
\begin{algorithmic}
\STATE {\bfseries Input:} dynamics parameters $\theta$, start time $0$, stop time $1$
\STATE $z_0 \sim q_0$  \COMMENT{Sampling from base distribution} 
\STATE $g_0 = \frac{\partial \ln q_0(z_0)}{\partial z_0}$ \COMMENT{Getting initial gradients}
\STATE $z_T, \ln q_\theta(z_T), \frac{\partial \ln q_\theta(z_T)}{\partial z_T} = \textnormal{Forw-Aug}(\theta, 0, T, z_0, \frac{\partial \ln q_0(z_0)}{\partial z_0})$ \COMMENT{Applying Alg.~\ref{algo1}}
\STATE $\ln p(z_T), \frac{\partial \ln p(z_T)}{\partial z_T} = \textnormal{Compute log-prob and derivative}$ \COMMENT{Compute log-prob and derivative}
\STATE $\_,  \blacktriangledown_\theta \ln q_\theta(z_T) = \textnormal{Rev.-mode der.}(\theta, T, 0, z_T, \frac{\partial \ln q_\theta(z_T)}{\partial z_T}-\frac{\partial \ln p(z_T)}{\partial z_T})$ \COMMENT{Applying Alg.~1 from \citet{chen2018neural}}
\STATE {\bfseries Return:} $z_T, \ln q_\theta(z_T), \blacktriangledown_\theta \ln q_\theta(z_T)$ \COMMENT{Return gradients}
\end{algorithmic}   
\end{algorithm*}

\begin{figure*}[ht]
\begin{center}
\includegraphics[width=\textwidth]{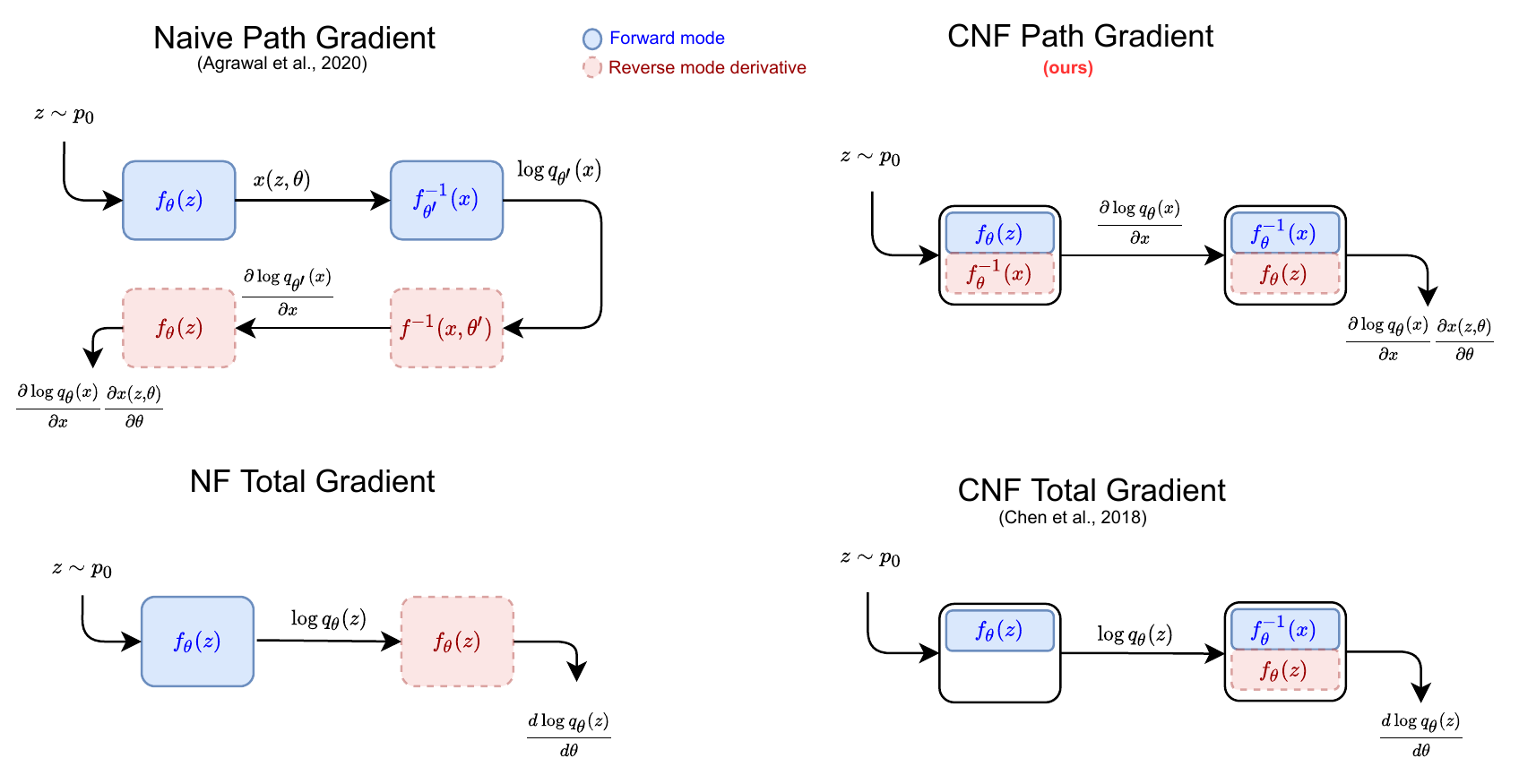}
\caption{Comparing the gradient computation for discrete as well as continuous normalizing flows}
\label{fig:full-algo-sketch}
\end{center}
\vskip -0.2in
\end{figure*}
\newpage

\section{Path gradient for CNFs}
\label{app:CNF-adjoint}

\begin{theorem}
The derivative $\tfrac{\partial \ln q_\theta(z_T)}{\partial z_T}$ can be obtained by solving the initial value problem
\begin{align}
    \frac{d}{dt}   \frac{\partial \ln  q_\theta(z_t)}{\partial z_t} = 	&-\frac{\partial \ln q_\theta(z_t)}{\partial z_t}^\intercal  \frac{\partial f_\theta(z_t,t)}{\partial z_t} \nonumber \\ & \hspace{1em} - 	 \partial_{z_t} \textrm{tr}\left(\frac{d f_\theta(z_t, t)}{d z_t} \right) \,, \label{eq:app-pathgradode}
\end{align}
with initial condition
\begin{align*}
    \frac{\partial \ln q_\theta(z_0)}{\partial z_0} =  \frac{\partial \ln q_Z(z_0)}{\partial z_0} \,.
\end{align*}
\end{theorem}

\begin{proof}
We define a generalized adjoint state by 
\begin{align}
    \alpha_t \equiv \frac{ d \ln q_\theta(z_t)}{d z_t}  \label{eq:gen-adjoint}  
\end{align}
Expanding the logarithm on the right-hand-side, we obtain
\begin{align*}
	 \alpha_t &= 	\frac{ d \left( \ln q_\theta(z_{t-\epsilon}) + 	\int_{t-\epsilon}^t\frac{\partial \ln q_\theta(z_{t})}{\partial t} \textrm{d} t \right)}{d z_{t}}  		  \\
 &= 	\frac{ d \ln q_\theta(z_{t-\epsilon}) }{d z_{t-\epsilon}}  		\frac{ d z_{t- \epsilon}}{d z_t}  + 	\frac{d \int_{t-\epsilon}^t\frac{\partial \ln q_\theta(z_{t})}{\partial t} \textrm{d} t}{d z_t}
 \end{align*}
Expanding the last summand up to linear order in $\epsilon$, we obtain 
 \begin{align*}
\alpha_t &= 	\frac{ d \ln q_\theta(z_{t-\epsilon}) }{d z_{t-\epsilon}}  		\frac{ d z_{t- \epsilon}}{d z_t} + 	\frac{d \frac{\partial \ln q_\theta(z_{t})}{\partial t}}{d z_t} \epsilon + \mathcal O(\epsilon^2)\\
 &= 	 \alpha_{t-\epsilon}  \frac{ d z_{t- \epsilon}}{d z_t} + \tilde \alpha_t \epsilon + \mathcal O(\epsilon^2) \, ,
  \end{align*}
 where we have used the definition of the generalized adjoint state $\alpha_t$ \eqref{eq:gen-adjoint} and have defined
\begin{align}
    	 \tilde \alpha_t = \frac{d\frac{\partial \ln q_\theta(z_{t})}{\partial t}}{d z_t} = - 	\frac{ \partial \textrm{tr}\left(\frac{df(t)}{d z_t} \right) }{\partial z_{t}} \label{eq:gen-adjoint-trace} \, .
\end{align}
Using $z_{t - \epsilon} = z_t - \epsilon f(z_t,t,\theta) + \mathcal O(\epsilon^2)$, we conclude
\begin{align*}
 \alpha_t &=  \alpha_{t - \epsilon} 	\frac{d }{d z_t}\Big(z_t - \epsilon 	f(z_t,t,\theta) \Big) + \tilde \alpha_t \epsilon + \mathcal O(\epsilon^2) \\ 
 &=  \alpha_{t - \epsilon} 	- \epsilon  \alpha_{t - \epsilon} \frac{\partial  f(z_t,t,\theta) }{\partial z_t} +  \epsilon \tilde \alpha_t + \mathcal O(\epsilon^2) \,.
  \end{align*}

The definition of the derivative then implies the following ODE
 \begin{align*}
	\frac{ d \alpha_t}{d t} &= \lim_{\epsilon \to 0} 	\frac{ \alpha_t - \alpha_{t - \epsilon}}{\epsilon} \\
 &= \lim_{\epsilon \to 0} 	\frac{- \epsilon 	 \alpha_{t - \epsilon} \frac{\partial f(z_t,t,\theta)}{\partial z_t}+ \tilde \alpha_t \epsilon  }{\epsilon} \\
&= - \alpha_t \frac{\partial f(z_t,t,\theta)}{\partial z_t} + \tilde \alpha_t  \, .
\end{align*}
From the definition of generalized adjoint state $\alpha_t$ \eqref{eq:gen-adjoint} and of $\tilde \alpha_t$ \eqref{eq:gen-adjoint-trace}, we arrive at the claim of the theorem 
\begin{align*}
\frac{d}{dt}   \frac{\partial \ln  q_\theta(z_t)}{\partial z_t} &= 	-\frac{ d \ln q_\theta(z_t)}{d z_t}   \frac{\partial f(z_t,t,\theta)}{\partial z_t} - 	\frac{ \partial \textrm{tr}\left(\frac{df(t)}{d z_t} \right) }{\partial z_{t}} \, .
\end{align*}

\end{proof}

By evolving the generalized adjoint state $\alpha_t$ from $t=0$ to $t=T$, we obtain the first term of the path gradient $\frac{\partial \ln q_\theta(x)}{\partial x} \frac{\partial x(z, \theta)}{\partial \theta}$. After computing this term in the forward pass, we can compute the whole path gradient as if the CNF were a standard Neural ODE, i.e. without taking the divergence term into account. 
The Hutchinson trace estimator can straightforwardly be applied to the trace independently of the adjoint state propagation.


\section{Experimental details}
\label{app:exp-details}
\subsection{Variational Autoencoder}
To reproduce the experiments of \citet{grathwohl2018scalable}, we implemented our proposed Algorithm~\ref{algo2} for the path gradient estimator  in their codebase. 

We train a VAE with an instance of Ffjord \citep{grathwohl2018scalable} between the en- and decoder.
As in the original experiments, the encoder network also outputs a low-rank update $\hat U(x) \hat V(x)^T $ to the  global weight matrix $W$ of the parameters of the flow as well as a bias vector $\hat b(x)$.
Both VAE networks are seven-layer networks with the architecture in the experiments by \citet{vdberg2018sylvester}, using gated and transposed convolutions for the encoder and decoder respectively.
For the baseline total gradient estimators, we used the results as reported in \citet{grathwohl2018scalable}. For the largest datasets MNIST and Omniglot the hyper-parameters for the path gradient estimators, we searched close to the best hyper-parameters for the total gradient, since the full grid is quite expensive to search. 
\begin{itemize}
    \item MNIST: two subsequent CNFS each with  a two hidden layer neural network with softplus, width 1024 and weight matrix updates of rank 64. Patience 35
    \item Omniglot: Five subsequent CNFs with a two hidden layer network with softplus, 512 hidden dimensions and weight matrix update of rank 64. Patience 35

\end{itemize}
For Caltech and Freyfaces, we used gridsearch over all the configurations which  \citet{grathwohl2018scalable} also used for finding the best hyper-parameters.
\begin{itemize}
    \item Caltech: Two subsequent CNF with a two hidden layer network with tanh, width 2048 hidden dimensions and weight matrix update of rank 20. Patience was 100.
    \item Freyfaces: Two CNFs with two hidden layers with softplus, width 512 and weight matrix update of rank 20. Patience was 100.
\end{itemize}
Training was done with a learning rate of .001, the Adam optimizer~\citep{kingma2014adam}, batch-size 100. The ODE solver was Dopri5. Training was done with a single GPU (A100).
  
While a single epoch generally took longer to train when using the path gradient estimator, the faster convergence (in epochs) meant that the training generally was comparable to training with the total gradient estimator.
Due to the adaptive step size of the ODE solver, the duration of one forward step of the CNF varied from run to run.
During the experiments, the addition of path gradient to the augmented dynamics did not have a marked effect on the number of function evaluations (NFE).
For different random seeds for the same hyper-parameters, the path gradient runs the NFE fluctuated. This sometimes led to the effect, that an epoch with the path gradient estimator was quicker than with the total gradient. Due to the increased computation time per function evaluation however, the total gradient mostly used less wall-time per epoch. 

Table~\ref{tab:loglike} shows the negative test log likelihood for the path and total gradient estimator. 

 \begin{table}[t]
 \caption{Negative test log likelihood (nll) for path and total gradient estimator. \textbf{Lower is better.} Per image, we used 5000 importance samples (2000 for Caltech) to estimate the nll.}
  \label{tab:loglike}
\vskip 0.15in
\begin{center}
\begin{small}
\begin{sc}
\begin{tabular}{@{}l|rr}
   Dataset  & Path & Total\\
  \toprule
        MNIST & 79.87 $\pm$ .14 & 80.10 $\pm$ .13 \\
        Omniglot & \textbf{92.64} $\pm$ .13 & 93.43 $\pm$ .06 \\
        Caltech Silhouettes & \textbf{91.65} $\pm$ .01 & 93.06 $\pm$ .02  \\
        Frey Faces (nll bpd) & 4.28 $\pm$ .03 & 4.36 $\pm$ .06 
    \end{tabular}
\end{sc}
\end{small}
\end{center}
\vskip -0.1in
\end{table}
\subsection{Lattice Field Theory}
For the physics application we mirrored the experiments of \citet{haan2021scaling}.
We only looked at lattice sizes where the effective sampling size could be improved, i.e. $L \in \{12,20,32\}$.
For training we chose a batch size of 5000 with a learning rate of 0.0005 and the ADAM optimizer~\citep{kingma2014adam}. 
The ODE solver was Runge-Kutta 4 with 50 steps.\\
Using these hyperparameters we were able to reproduce and even improve the baseline, which was optimized using the total gradient estimator.
Since the usage of the path gradient estimator increases the runtime, we let the baseline run for 20$\%$ more epochs than the path gradient estimator. In practice the quadratic complexity $\mathcal O(d^2)$ of the equivariant flow proposed by \citet{haan2021scaling}, the difference in computation time decreases with increasing dimensionality of the lattice. \\
For $L \in \{12,20\}$ the baseline was run for 24k batches and the path gradient for 20k batches. For $L = 32$, since the models seem to have converged by then, we only trained for  15.4k and 12k respectively.
Optimization for $L=32$ showed some instable behavior for the baseline. For the reported results for the total gradient estimator, we only took runs into account, that did not diverge.
Training was done with a single GPU (A100).

Since the energy of the LQFT's target density is known in closed-form (i.e. is not learnt from data), one can also use high-precision tools of theoretical physics, see \cite{nicoli2021estimation}, to estimate the reverse KL divergence directly. Analogously to the effective sampling size, the reverse KL divergence is consistently lower for the models trained with path-gradients, see Table~\ref{tab:phi4KL}. 

\begin{table}[t]
  \caption{Reverse KL divergence estimated as described in \cite{nicoli2021estimation} using 4M samples. {\bf Lower is better.}}
  \label{tab:phi4KL}
\vskip 0.15in
\begin{center}
\begin{small}
\begin{sc}
 \begin{tabular}{@{}l|rr} 
        Lattice size & Path & Total  \\
         \toprule
        12x12 &  \textbf{.0016} $\pm$ .0005 &  .0113 $\pm$ .0027\\
        20x20 &  \textbf{.0123} $\pm$ .0008 &  .0415 $\pm$ .0087\\
        32x32 &  \textbf{.0498} $\pm$ .0049 & .1833 $\pm$ .0643\\
    \end{tabular}
\end{sc}
\end{small}
\end{center}
\vskip -0.1in
\end{table}

\section{Runtime Analysis}
In order to analyse the additional computational cost of computing the path gradients as described in Alg.~\ref{algo2}, we tested the increase in runtime over different number of dimensions.
We fixed the parameters of the hidden layers in a vanilla CNF and recorded the time that each "ODESolve" call took.
This ignores e.g. the additional runtime, that duplicating the model yields in the naive path gradient estimator.
Figure~\ref{fig:runtime} shows the increase in runtime analysed with 24 repetitions. We can see that it yields an increase of roughly $ 14\%$.
In practice the complexity of the model increases with the number of dimensions of the sample. Furthermore there are other operations that also scale with the problem complexity, so generally we can expect a lesser impact on the runtime than shown in Figure~\ref{fig:runtime}.

\end{document}